\PassOptionsToPackage{cleveref}{jmlrutils}
\documentclass{alt2027}

\jmlrproceedings{}{}
\jmlrvolume{}
\jmlrpages{}
\jmlryear{}

\usepackage[T1]{fontenc}
\usepackage{times}
\usepackage{mathtools}
\usepackage{enumerate}

\numberwithin{theorem}{section}
\aliascntresetthe{lemma}
\aliascntresetthe{proposition}
\aliascntresetthe{corollary}
\aliascntresetthe{definition}
\aliascntresetthe{remark}
\newaliascnt{fact}{theorem}
\newtheorem{fact}[fact]{Fact}
\aliascntresetthe{fact}
\crefname{fact}{fact}{facts}
\Crefname{fact}{Fact}{Facts}
\newaliascnt{claim}{theorem}

\aliascntresetthe{claim}
\crefname{claim}{claim}{claims}
\Crefname{claim}{Claim}{Claims}
\numberwithin{equation}{section}

\renewenvironment{proof}[1][\proofname]
  {\par\noindent{\bfseries\upshape #1\ }\ignorespaces}
  {\jmlrQED}

\newcommand{\calH}{\mathcal{H}}
\newcommand{\sse}{\subseteq}
\newcommand{\eps}{\varepsilon}
\newcommand{\E}{\mathbf{E}}
\renewcommand{\emptyset}{\varnothing}
\newcommand{\cX}{\mathcal{X}}
\newcommand{\cY}{\mathcal{Y}}
\newcommand{\cQ}{\mathcal{Q}}
\newcommand{\cD}{\mathcal{D}}

\newcommand{\cR}{\mathcal{R}}
\renewcommand{\Pr}{\mathbf{Pr}}
\newcommand{\indic}[1]{\mathbf{1}\{#1\}}

\DeclareMathOperator{\Unif}{Unif}
\DeclareMathOperator{\VC}{VC}
\DeclareMathOperator{\Ndim}{Ndim}
\DeclareMathOperator{\depth}{depth}
\DeclareMathOperator*{\argmin}{arg\,min}

\newcommand{\mQ}{m_{\cQ}}
\mathchardef\mhyphen="2D

\title[PAC Learning with Query Feedback]{PAC Learning with Query Feedback}
\altauthor{
 \Name{Chase Hutton} \Email{chutton6@umd.edu}\\
 \addr University of Maryland
 \AND
 \Name{Adam Melrod} \Email{amelrod@math.harvard.edu}\\
 \addr Harvard University
 \AND
 \Name{Han Shao} \Email{hanshao@umd.edu}\\
 \addr University of Maryland
}
\hypersetup{pdftitle={PAC Learning with Query Feedback},pdfauthor={Chase Hutton, Adam Melrod, Han Shao}}

\begin{document}

\maketitle
\begin{abstract}
We study realizable multiclass PAC learning with \emph{query
feedback}, where each training example provides a single binary
response indicating whether its true label belongs to a chosen
subset from a fixed query family. Our main contribution is the
\emph{query dimension}, a combinatorial parameter that captures
both the structure of the hypothesis class and the ability of
the available queries to distinguish labels. This parameter
generalizes the bandit DS dimension, which it recovers under
singleton queries. We
establish a sample-complexity lower bound of $\Omega(D/\varepsilon)$
and an upper bound of $\widetilde{O}(D^2/\varepsilon)$, suppressing
factors polylogarithmic in the number of labels and the inverse
failure probability. When the available queries can split every
candidate set into two parts of comparable size, we improve the
upper bound to $\widetilde{O}(D/\varepsilon)$, yielding a
characterization up to logarithmic factors.
\end{abstract}

\section{Introduction}\label{sec:introduction}
In the multiclass PAC learning problem, a learner is tasked with classifying objects into one of $K$ possible labels. A natural variant of this problem arises when, upon making a prediction on a training example, the learner receives only binary feedback indicating whether its prediction was correct, rather than being given the true label. This setting, known as multiclass PAC learning with bandit feedback, was introduced and studied in the PAC framework by \citet{daniely2011multiclass,daniely2015multiclass}, and its sample complexity was recently characterized, up to logarithmic factors, by \citet{bds} in terms of a combinatorial parameter called the bandit DS dimension.

Bandit feedback models situations in which a learner chooses an action and observes only whether it succeeds. As motivation, \citet{bds} describe a treatment scenario in which each patient is assumed to respond to exactly one of the available drugs. Administering a drug then tests a single candidate label: the learner observes whether that drug was effective, but an unsuccessful treatment does not reveal which drug would have worked.

Binary feedback need not be limited to testing a single candidate label: an action may instead test whether the true label belongs to a subset of labels. For example, suppose the goal is to identify which of $K$ bacterial species caused a patient's infection, and each available drug is effective against a known subset of these species. Under the simplifying assumption that treatment succeeds exactly when the responsible bacterium belongs to that subset, administering a drug reveals whether the true label lies in the corresponding subset. This leads naturally to what we call \emph{query feedback}: the learner asks whether the true label belongs to a chosen subset of the label space, and the family of subsets it may ask about is fixed in advance.

Formally, we study the following problem. Let $\calH \sse \cY^{\cX}$ be a hypothesis class over an instance space $\cX$ and a finite label space $\cY$ with $|\cY|=K$, and let $\cQ \sse 2^{\cY}$ be a fixed family of available queries. We consider the realizable setting, in which an unknown target $h^*\in\calH$ determines the label of every instance. Training proceeds over $m$ rounds. On round $t$, the learner observes an instance $x_t$ drawn from an unknown distribution $\mathcal D$ on $\cX$, chooses a query $Q_t\in\cQ$, which may depend on $x_t$ and on the history of the interaction, and receives the single query response $b_t=\indic{h^*(x_t)\in Q_t}$ as feedback. After $m$ rounds, the learner outputs a classifier $\widehat h\colon\cX\to\cY$, and the goal is that its error $\Pr_{X\sim \mathcal D}[\widehat h(X)\neq h^*(X)]$ is at most $\eps$ with probability at least $1-\delta$. We denote by $\mQ(\calH;\eps,\delta)$ the smallest number of rounds for which some learner achieves this guarantee for every distribution $\mathcal D$ and every target $h^*\in\calH$.

Observe that taking $\cQ_{\mathrm{sing}}:= \{\{y\} : y \in \mathcal Y\}$ recovers multiclass PAC learning with bandit feedback. For more general $\cQ$, the structure of the available queries can make learning considerably more efficient. Whereas under bandit feedback an incorrect prediction eliminates a single candidate label, a query that divides the candidate labels into two groups of comparable size rules out many of them at once. Not every query about a large subset is useful however. Indeed, a query containing all remaining candidates reveals nothing. What matters is whether the allowed queries can separate labels that the learner still needs to distinguish. In this work we study how this ability governs the number of examples required to learn:
 
\begin{center}
\emph{Given a hypothesis class $\calH\sse\cY^{\cX}$ and a family of legal queries $\cQ\sse2^{\cY}$, how many training rounds are necessary and sufficient for PAC learning $\calH$ when the learner receives only query feedback from $\cQ$?}
\end{center}

The bandit DS dimension of \citet{bds} is a natural starting point to answering this question. As such, let us briefly recall its definition. The main combinatorial objects used to define the bandit DS dimension are \emph{pseudo-boxes}: finite sets $W$ of label vectors, realized by $\calH$ on instances $x_1,\dots,x_d$, such that every vector in $W$ has, for each coordinate $i$, at least $N_i$ neighbors in $W$ that differ from it only at $x_i$. The dimension is then the largest total $\sum_iN_i$ that $\calH$ realizes in this way. The dimension essentially counts locally indistinguishable alternatives: hypotheses that differ from one another only at a single instance which bandit feedback can rule out only one at a time. For general query families, however, this metric is no longer enough. The same set of candidate labels can be difficult to distinguish when queries test one label at a time, but easy when suitable subset queries are available. A dimension for our setting must therefore account both for the
hypothesis class $\calH$, which determines which labels the learner needs
to distinguish, and for the query family $\cQ$, which determines how many
rounds it takes to distinguish them. Our main contribution is to
identify such a dimension and establish upper and lower bounds on
the sample complexity in terms of it.

\section{Overview and Main Results}

We call this parameter the \emph{query dimension} $D_{\cQ}(\calH)$.
To define it, we first assign each nonempty candidate set
$C\sse\cY$ a normalized refinement cost $\rho_{\cQ}(C)$, which
measures the depth of an available query tree relative to the
refinement it achieves. We then aggregate these costs across
instances in the manner of the bandit DS dimension. We give an
informal description here and refer the reader to
\Cref{sec:refinement} for the formal definitions.
  
\paragraph{The cost of distinguishing labels.}
Fix a finite set $C\sse\cY$ with $|C|\ge2$. We model an adaptive strategy for distinguishing the labels in $C$ as a decision tree. Specifically, we define a \emph{$\cQ$-decision tree} to be a binary tree whose internal nodes are labeled by queries from $\cQ$, with the two branches out of a node labeled $A$ corresponding to the two possible responses. A label $y\in C$ follows the branch indexed by $\indic{y\in A}$ at each node it visits. Labels that reach the same leaf remain indistinguishable, so the leaves induce a partition $\Pi_{T}(C)$ of $C$ into cells. We say that the tree \emph{completely refines} $C$ if every cell is a singleton. Complete refinement is not always achievable, as two labels may be indistinguishable under every query in $\cQ$.
 
The minimum depth of a tree that completely refines $C$ is a natural measure of difficulty. For our purposes, however, it is useful to account for partial refinement as well. A tree $T$ that induces $|\Pi_{T}(C)|$ cells has created $|\Pi_{T}(C)|-1$ new cells out of the $|C|-1$ needed for complete refinement, and we regard the ratio of these two numbers as the fraction of complete refinement it achieves. The \emph{query-refinement cost} of $C$ is the depth of a tree divided by the fraction of refinement it achieves, minimized over trees:
\[
    \rho_{\cQ}(C)
    :=
    \min_{T}\;
    \depth(T)\,
    \frac{|C|-1}{|\Pi_{T}(C)|-1},
\]
where the minimum ranges over $\cQ$-decision trees that induce at least two cells on $C$. For a tree that completely refines $C$ the ratio is simply its depth; for a partial refinement, the depth is scaled up by the reciprocal of the progress made. We set $\rho_{\cQ}(C):=0$ when $|C|=1$, and $\rho_{\cQ}(C):=\infty$ when no query in $\cQ$ separates any two labels of $C$.

This normalization has a useful interpretation in terms of prediction. Consider a learner that asks $t$ queries about an unknown label $y\in C$ and then guesses $y$. Each query may depend on the responses to the previous ones, so the queries the learner can ask form a $\cQ$-decision tree of depth $t$, and its guess is a function of the responses, that is, of the leaf that $y$ reaches. Every deterministic learner is thus described by a $\cQ$-decision tree together with a guessed label at each leaf, and a randomized learner is simply a mixture of such learners. Now consider a leaf of the learner's tree and the cell of labels in $C$ that reach it. The learner cannot distinguish these labels, so it guesses the same label for all of them and is correct for at most one. The learner is therefore correct on at most $|\Pi_{T}(C)|$ of the labels in $C$. By the definition of $\rho_{\cQ}(C)$, a tree of depth $t$ induces at most $1+t(|C|-1)/\rho_{\cQ}(C)$ cells, so if $y$ is drawn uniformly from $C$, the learner guesses it correctly with probability at most
\begin{equation}
\label{eq:intro-prediction}
    \frac{|\Pi_{T}(C)|}{|C|}
    \le\frac1{|C|}+\frac{t}{\rho_{\cQ}(C)}\cdot\frac{|C|-1}{|C|}
    \le\frac12+\frac{t}{\rho_{\cQ}(C)}.
\end{equation}
Thus, a learner that asks substantially fewer than $\rho_{\cQ}(C)$ queries cannot identify a uniformly chosen label from $C$ with high probability.
 
\paragraph{From a set of labels to a hypothesis class.}
The query dimension aggregates query-refinement costs across instances using the local structure of the pseudo-boxes of \citet{bds}. A \emph{witness} consists of distinct instances $x_1,\dots,x_d$ and a nonempty finite set $W$ of label vectors realized by $\calH$ on them, that is, $W\sse\calH|_{x_{1:d}}:=\{(h(x_1),\dots,h(x_d)):h\in\calH\}$. For $v\in W$ and $i\in[d]$, define the \emph{$i$-line} through $v$ by $L_i^W(v):=\{u_i:u\in W,\ u_{-i}=v_{-i}\}.$ This is the set of candidate labels at $x_i$ whose label vectors remain in $W$
after the labels at all other selected instances have been fixed to $v_{-i}$. Since the corresponding vectors agree at every other coordinate, feedback at those instances cannot distinguish between these candidates. The cost $\rho_{\cQ}(L_i^W(v))$ therefore measures a local ambiguity whose resolution requires queries at $x_i$. To obtain a coordinate cost that is uniform over the choice of target vector in $W$, we define
\[
    r_i(W):=\min_{v\in W}\rho_{\cQ}\bigl(L_i^W(v)\bigr).
\]
Every $i$-line in $W$ thus has query-refinement cost at least $r_i(W)$. These coordinate costs contribute additively, reflecting the fact that the ambiguity at each instance must be resolved independently through queries at that instance. Accordingly, we define the \emph{query dimension} as the supremum of the total cost over all witnesses:
\[
    D_{\cQ}(\calH)
    :=
    \sup_{W \sse \calH|_{x_{1:d}}}\ \sum_{i=1}^d r_i(W).
\]
This definition retains the additive structure of the bandit DS dimension while replacing local neighbor counts by query-refinement costs. Indeed, under singleton queries,
$\rho_{\cQ_{\mathrm{sing}}}(L)=|L|-1$, so the cost of an $i$-line through $v$ is exactly the number of vectors in $W$ that differ from $v$ only at coordinate $i$. Consequently, $r_i(W)$ is the minimum such neighbor count, as in the pseudo-boxes of \citet{bds}, and
$D_{\cQ_{\mathrm{sing}}}(\calH)=\operatorname{BDS}(\calH)$. For a general query family, the contribution of a coordinate depends instead on how efficiently the available queries distinguish the labels on its lines. This cost can be substantially smaller than the number of alternatives, allowing the dimension to reflect both the local structure of the hypothesis class and the information provided by the query family.

\paragraph{Sample-complexity bounds.}
Our main result bounds the sample complexity in terms of the query
dimension.

\begin{theorem}[Sample complexity with query feedback, informal]
\label{thm:intro-main}
Let $\calH\sse\cY^{\cX}$ be a hypothesis class and
$\cQ\sse2^{\cY}$ a query family, with
$D:=D_{\cQ}(\calH)<\infty$ and $K:=|\cY|$.
For every $0<\eps,\delta\le1/8$,\footnote{The lower bound assumes
that $(\calH,\cQ)$ is non-trivial; see \Cref{def:nontrivial}.}
\begin{equation}
\label{eq:intro-main}
    \Omega\!\left(\frac{D+\ln(1/\delta)}{\eps}\right)
    \le \mQ(\calH;\eps,\delta)
    \le \widetilde{O}\!\left(\frac{D^2}{\eps}\right),
\end{equation}
where $\widetilde{O}$ hides factors polylogarithmic in $K$ and
$1/\delta$.
\end{theorem}

\paragraph{Comparison of dimensions.}
This learning problem has less informative feedback than in traditional multiclass PAC learning. On the other hand, depending on the queries allowed, the feedback can be even incomparable to bandit feedback. Nevertheless, we prove a comparison on the level of the relevant combinatorial dimensions (Proposition \ref{prop:dimension-comparison}):
\[
    \mathrm{DS}(\calH) \leq D_{\cQ}(\calH) \leq \mathrm{BDS}(\calH).
\]

\paragraph{Balanced queries.} When the available queries can split every candidate
set into two parts of comparable size, we improve the upper bound to
$\widetilde O(D_{\cQ}(\calH)/\eps)$. This gives a characterization up to logarithmic factors; see \Cref{sec:examples}.

 \subsection{Overview of the proof of Theorem \ref{thm:intro-main}}

 \paragraph{Lower bound.}
We follow the lower-bound construction of \citet{bds}. In the bandit setting, the argument at a single instance relies on the fact that an unsuccessful prediction eliminates only one candidate label. This need not hold for general queries, which may eliminate many candidates at once, so we use \Cref{eq:intro-prediction} in its place: a learner that asks $n$ queries about a label drawn uniformly from $C$ identifies it with probability at most $1/2+n/\rho_{\cQ}(C)$.

Fix a witness $W$ on instances $x_1,\dots,x_d$ with coordinate costs $r_1,\dots,r_d$ and total cost $R=\sum_ir_i$. We consider distributions supported on these instances and choose the target vector uniformly from $W$. The obstacle to applying \Cref{eq:intro-prediction} at an instance $x_i$ is that the learner's queries there may depend on feedback it received at other instances. We remove this obstacle by revealing to the learner the target's labels at all witness instances other than $x_i$. This can only help the learner, but it leaves the label at $x_i$ uniform on an $i$-line $C$ with $\rho_{\cQ}(C)\ge r_i$, and once the sampled instances and the learner's randomness are also fixed, every response at another instance can be computed from the revealed labels. The learner's prediction at $x_i$ is then determined by the responses at $x_i$ alone, so it is given by a $\cQ$-decision tree on $C$ whose depth is the number $n$ of visits to $x_i$, and \Cref{eq:intro-prediction} bounds its probability of being correct by $1/2+n/r_i$. Thus the learner errs at $x_i$ with constant probability whenever the number of visits to $x_i$ is a small multiple of $r_i$.

We then choose the sampling distribution to make these visits infrequent, while assigning enough mass to the difficult instances that the learner cannot ignore them. We select an anchor of minimum coordinate cost and restrict $W$ to the vectors with a fixed label at the anchor, which leaves the lines at the other coordinates unchanged. The anchor receives mass $1-8\eps$, and the remaining $8\eps$ is divided among the other instances in proportion to their coordinate costs; since the anchor has the smallest cost, these instances have total cost at least $R/2$. Visits to the anchor reveal nothing about the target, and in $m$ rounds a non-anchor instance $x_i$ is visited about $m\eps r_i/R$ times in expectation. If $m$ is a sufficiently small multiple of $R/\eps$, the learner therefore errs at each non-anchor instance with constant probability, which gives expected error at least $3\eps$ on instances of total mass $8\eps$, and hence error greater than $\eps$ with constant probability for some target. Choosing a witness of total cost $D$ gives the lower bound $\Omega(D/\eps)$.

\paragraph{Upper bound.}
Our learner repeatedly halves a list of candidate labels at every
instance $x$, starting from $C(x)=\calH(x):=\{h(x):h\in\calH\}$.
Each stage constructs at each instance a distribution over legal queries based on the
current list. Using fresh examples and one sampled query per example,
the learner trains a binary predictor $\cX \times \cQ \to \{0,1\}$ of the target's query responses.
For each $x$, it then selects the candidate whose responses best match
this predictor and retains candidates whose responses are close to
those of the selected candidate. Both comparisons measure disagreement
under the same query distribution. The distributions and
retention threshold ensure that every nonsingleton list remains
nonempty and shrinks by at least half; singleton lists are unchanged.
After $\lceil\log_2 K\rceil$ stages, the learner outputs the sole
remaining label at each instance.

The query distribution at each instance in a stage is induced by a particular query tree. Write $D:=D_{\cQ}(\calH)$.
Every nonsingleton $C\sse\calH(x)$ admits a \emph{halving tree}
of depth at most $\tau=O(D\log K)$, whose leaves contain at most
$|C|/2$ candidates (\Cref{lem:rho-to-tau}). Weighting the tree's
queries by the number of candidates that reach each node yields
a distribution $\lambda_C$ with a useful separation property.
Represent a label $y$ by its \emph{query trace}, the binary function
$A\mapsto\indic{y\in A}$, and measure trace disagreement by
$L_{\lambda_C}(y,z):=\Pr_{A\sim\lambda_C}[
\indic{y\in A}\neq\indic{z\in A}]$. Then
\[
    \left|\left\{z\in C:
        L_{\lambda_C}(y,z)\le\frac1{2\tau}
    \right\}\right|\le\frac{|C|}{2},
    \qquad y\in C,
\]
by \Cref{lem:halving-query-distribution}. Thus a small neighborhood
of any candidate trace contains at most half the list. Choosing a
candidate close to the target trace and retaining its neighborhood
halves the list without discarding the target.

The other ingredient is binary PAC learning under these same query
distributions. The observed responses are realized by the class
$\cR_{\cQ}(\calH):=\{r_h:h\in\calH\}$ of functions
$r_h(x,A):=\indic{h(x)\in A}$ on instance--query pairs. A binary predictor's error is exactly its average trace
disagreement under the query distributions used for list reduction.
Selecting the closest candidate trace then bounds the probability
mass on which a stage newly discards the target by $O(\tau)$ times
this binary error (\Cref{lem:one-step-list-halving}).

Finally, the response class has VC dimension
$d=O(D\log K)$ (\Cref{lem:response-dim}). Dividing the error and
failure among the $O(\log K)$ stages gives total sample
complexity
\[
    \widetilde O\!\left(\frac{\tau d}{\eps}\right)
    =\widetilde O\!\left(\frac{D^2}{\eps}\right).
\]
The two factors of $D$ have distinct sources: the response-class
dimension controls the cost of learning query responses, while the
tree depth controls how accurately those responses must be learned
to retain the target during list reduction. The full proof appears
in \Cref{sec:quadratic-upper}.

\section{Related Work}\label{sec:related-work}
\paragraph{Multiclass learning with bandit feedback.}
Bandit feedback for multiclass prediction was introduced in the online setting by \citet{kakade2008efficient}, whose Banditron is a bandit version of the multiclass perceptron. \citet{daniely2013price} quantified what this feedback costs an online learner: a factor of $O(K\log K)$ on the number of mistakes in the realizable case and $\widetilde\Theta(\sqrt K)$ on the regret in the agnostic case. The PAC formulation is due to \citet{daniely2011multiclass,daniely2015multiclass}, who bounded its sample complexity in terms of the Natarajan dimension and showed that withholding the label can cost a factor of order $K$. These bounds leave a gap for general classes, and a characterization of the sample complexity of an arbitrary class came only recently: \citet{bds} show that it is governed, up to logarithmic factors, by the bandit DS dimension, and give an algorithm that attains the bound by learning a sequence of progressively shorter lists of candidate labels.

The paper of \citet{bds} is the immediate predecessor of ours. The witnesses in our lower bound are their pseudo-boxes and our hard distribution is adapted from theirs; what changes is the quantity charged to a coordinate. A pseudo-box charges a coordinate by the number of alternative labels it contains, which is the right measure when a query can eliminate one label at a time. The query dimension charges it by the cost of separating those alternatives with the queries that are actually available. The two agree when every query is a singleton (\Cref{prop:singleton-bds}), and the cost can be far smaller than the count when the queries are more informative.

Outside the realizable PAC setting, \citet{erez2024real} determine the minimax regret of online bandit classification with a finite class over $T$ rounds, $\widetilde\Theta(\min\{|\calH|+\sqrt T,\sqrt{KT\log|\calH|}\})$, and \citet{erez2024fast} give an agnostic PAC learner for finite classes that uses $O((\mathrm{poly}(K)+\eps^{-2})\log(|\calH|/\delta))$ examples, so the factor of $K$ paid in the realizable case disappears from the leading term as $\eps\to0$.

\paragraph{Multiclass learnability.}
For a finite label space, PAC learnability of a multiclass class is governed by its Natarajan dimension \citep{natarajan1989learning,bendavid1995characterizations}, up to a factor logarithmic in $K$. For label spaces of unbounded size the right parameter is the DS dimension of \citet{daniely2014optimal}, defined through pseudo-cubes, finite sets of label vectors in which every vector has an alternative in each coordinate. \citet{brukhim2022characterization} proved that the DS dimension characterizes multiclass PAC learnability, and \citet{charikar2023characterization} carried the characterization over to list learning, in which the learner outputs a short list that should contain the correct label.

\paragraph{Partial labels and label-subset queries.}
A query response is a candidate set for the label: a positive answer to $A$ places $h^*(x)$ in $A$, and a negative answer places it in $\cY\setminus A$. Learning from candidate sets is the subject of partial-label learning, also called superset-label learning \citep{cour2011learning,liu2014learnability}, and a complementary label \citep{ishida2017learning}, which names a class the example does not belong to, is exactly what a negative bandit response provides. In these models the candidate sets are produced by the annotator or the environment, and the analysis asks when the classifier can be recovered from them. Here the sets are produced by the learner, one per example, from a prescribed family, and the question is how many examples this costs.

Protocols in which the learner chooses its own label-subset questions have been studied mostly from the algorithmic side. \citet{hu2019active} let an active learner pick both an example and a binary question about a composite class in a label hierarchy, with further questions about the same example allowed; their evaluation is empirical. The closest protocol to ours is the streaming variant of active labeling in \citet{cabannes2022active}, which permits one membership query per arriving example, chosen from a fixed family of label subsets. Their concern is optimization: they show that such answers suffice to form unbiased stochastic gradients, prove minimax rates for median regression, and treat classification through surrogate losses. \citet{su2026risk} fix the querying rule instead, drawing the subset uniformly among those of a given size, and study unbiased and corrected estimators of the classification risk. None of these works asks how many examples a given hypothesis class requires under a given query family, which is the question studied here.

\section{Model}
\label{sec:setup}

Let $\cX$ be an instance space, let $\cY$ be a finite label space of size $K:=|\cY|\ge2$, and let $\calH\sse\cY^{\cX}$ be a nonempty hypothesis class. We fix a family $\cQ\sse2^{\cY}$ of available queries. We work in the realizable setting: an unknown target $h^*\in\calH$ determines the labels, and instances are drawn independently from an unknown distribution $\mathcal D$ on $\cX$.

\subsection{Learning with query feedback}
\label{sec:query-pac}

The error of a classifier $h\colon\cX\to\cY$ is its probability of disagreeing with the target on a fresh instance:
\[
    L_{\mathcal D}(h,h^*)
    :=
    \Pr_{X\sim \mathcal D}\bigl[h(X)\neq h^*(X)\bigr].
\]
The goal is to output a classifier with small error, despite observing only query responses during training.

\paragraph{Query feedback.}
On round $t$, the learner observes an instance $x_t\sim \mathcal D$, chooses a query $Q_t\in\cQ$, and receives the response
\[
    b_t=\indic{h^*(x_t)\in Q_t}.
\]
The query may depend on $x_t$, the preceding observations $(x_s,Q_s,b_s)_{s<t}$, and the learner's internal randomness. After $m$ rounds, the learner returns a classifier $\widehat h\colon\cX\to\cY$, which need not belong to $\calH$. We measure the number of rounds needed to achieve the usual PAC guarantee.

\begin{definition}[Realizable sample complexity]
\label{def:sample-complexity}
For $\eps,\delta\in(0,1)$, the sample complexity of $\calH$ with query feedback from $\cQ$, denoted by $\mQ(\calH;\eps,\delta)$, is the least nonnegative integer $m$ for which there exists a learner satisfying
\[
    \Pr\bigl[L_{\mathcal D}(\widehat h,h^*)>\eps\bigr]\le\delta
\]
after $m$ rounds, for every distribution $\mathcal D$ on $\cX$ and every target $h^*\in\calH$. The probability is over the sampled instances and the learner's internal randomness. We set $\mQ(\calH;\eps,\delta)=\infty$ if no such $m$ exists.
\end{definition}

\section{Combinatorial Dimension}
\label{sec:combinatorial-dimensions}

We define the query dimension and compare it with the DS and bandit DS
dimensions. We then bound two quantities used in the upper bound of
\Cref{sec:quadratic-upper}: identification depth and the VC dimension of the
extended response class. We conclude with a characterization up to logarithmic factors
for balanced query families. Proofs
not included here appear in \Cref{app:query-dimension-proofs}.

\subsection{Query refinement}
\label{sec:refinement}

We use decision trees to measure the information obtainable about an unknown
label through queries from $\cQ$.

\begin{definition}[$\cQ$-decision tree]
\label{def:tree}
A \emph{$\cQ$-decision tree} $T$ is a finite rooted binary tree whose
internal nodes are labeled by queries from $\cQ$, with the two outgoing edges
indexed by $0$ and $1$. At a node labeled by $A$, a label $y\in\cY$ follows
the edge indexed by $\indic{y\in A}$. The depth $\depth(T)$ is the
maximum number of internal nodes on a root-to-leaf path.
\end{definition}

For a nonempty $C\sse\cY$, let $\Pi_{T}(C)$ be the partition of $C$
into the nonempty sets of labels reaching the same leaf. We say that tree
\emph{completely refines} $C$ if every cell is a singleton. For $|C|\ge2$, a
tree inducing $q\ge2$ cells achieves a fraction $(q-1)/(|C|-1)$ of the
refinement needed to isolate every label. We measure its cost by dividing its
depth by this fraction.

\begin{definition}[Query-refinement cost]
\label{def:rho}
For a nonempty $C\sse\cY$, set $\rho_{\cQ}(C):=0$ if $|C|=1$. For $|C|\ge2$, define
\begin{equation}
\label{eq:rho-def}
    \rho_{\cQ}(C)
    :=
    \inf_{\substack{
        T\text{ a }\cQ\text{-decision tree}\\
        |\Pi_{T}(C)|\ge2
    }}
    \depth(T)\,
    \frac{|C|-1}{|\Pi_{T}(C)|-1},
\end{equation}
with value $\infty$ if no such tree exists.
\end{definition}

A complete refinement costs its depth, but a partial refinement may
have smaller normalized cost. The following bounds relate this cost
to the number of candidate labels.

\begin{lemma}
\label{lem:rho-basic}
For every $C\sse\cY$ with $|C|\ge2$,
\[
    \rho_{\cQ}(C)\ge\log_2|C|.
\]
If $\rho_{\cQ}(C)<\infty$, then $\rho_{\cQ}(C)\le|C|-1$. In particular, for
distinct $a,b\in\cY$, the cost $\rho_{\cQ}(\{a,b\})$ is $1$ if some query
separates $a$ and $b$, and $\infty$ otherwise.
\end{lemma}

\subsection{The query dimension}
\label{sec:dimension}
To define the query dimension, we use the witness structure of
the bandit DS dimension, with refinement costs in place of neighbor counts.
Write $\calH(x):=\{h(x):h\in\calH\}$ and, for distinct
instances $x_{1:d}=(x_1,\dots,x_d)$,
\[
    \calH|_{x_{1:d}}
    :=
    \{(h(x_1),\dots,h(x_d)):h\in\calH\}.
\]
A \emph{witness} for $\calH$ consists of such a sequence and a nonempty finite set
$W\sse \calH|_{x_{1:d}}$. For $v\in W$ and $i\in[d]$, let $v_{-i}$ denote $v$ with coordinate $i$
deleted. The \emph{$i$-line} through $v$ is
\[
    L_i^W(v)
    :=
    \{u_i:u\in W,\ u_{-i}=v_{-i}\}.
\]
It contains the labels at $x_i$ allowed by $W$ when the other coordinates are
fixed to those of $v$. We assign coordinate $i$ the minimum refinement cost of
its lines and sum these costs over coordinates:
\[
    r_i(W):=\min_{v\in W}\rho_{\cQ}\bigl(L_i^W(v)\bigr),
    \qquad
    r_{\cQ}(W):=\sum_{i=1}^d r_i(W).
\]

\begin{definition}[Query dimension]
\label{def:dimension}
The \emph{query dimension} of $\calH$ is
\begin{equation}
\label{eq:dimension}
    D_{\cQ}(\calH)
    :=
    \sup_{\substack{
        d\ge1,\ x_1,\dots,x_d\in\cX\text{ distinct}\\
        \emptyset\neq W\sse \calH|_{x_{1:d}},\ |W|<\infty
    }}
    r_{\cQ}(W),
\end{equation}
with value zero if there are no witnesses.
\end{definition}

We say that $\cQ$ \emph{separates $\calH$ pointwise} if, for every $x\in\cX$ and
every pair of distinct labels $a,b\in\calH(x)$, some query contains exactly one
of $a$ and $b$.

\begin{lemma}
\label{lem:query-dimension-basic}
For every $x\in\cX$ and every nonempty $C\sse \calH(x)$,
\[
    \rho_{\cQ}(C)\le D_{\cQ}(\calH).
\]
Consequently, if $D_{\cQ}(\calH)<\infty$, then $\cQ$ separates $\calH$ pointwise.
\end{lemma}

\begin{proof}
The set $C$ is a witness on the single instance $x$, with cost
$\rho_{\cQ}(C)$. An unseparated pair in $\calH(x)$ would therefore give infinite
query dimension by \Cref{lem:rho-basic}.
\end{proof}

\subsection{Comparison with existing dimensions}
\label{sec:dimension-comparison}

A DS witness, or \emph{pseudo-cube}, requires at least one alternative label
in every coordinate line and is assigned its number of coordinates
\citep{daniely2014optimal,brukhim2022characterization}. A BDS witness requires
at least $N_i\ge1$ alternatives in each $i$-line and is assigned the weight
$\sum_iN_i$ \citep{bds}. The corresponding dimensions maximize these
quantities over witnesses. Formal definitions appear in \Cref{app:dimensions}. The below proposition shows that the query dimension is lower bounded by the DS dimension and upper bounded by the bandit DS dimension.

\begin{proposition}[Comparison of dimensions]
\label{prop:dimension-comparison}
If $D_{\cQ}(\calH)<\infty$, then
\begin{equation}
\label{eq:dimension-comparison}
    \operatorname{DS}(\calH)
    \le D_{\cQ}(\calH)
    \le \operatorname{BDS}(\calH)
    \le (K-1)\operatorname{DS}(\calH).
\end{equation}
More generally, these inequalities hold, allowing infinite values, whenever
$\cQ$ separates $\calH$ pointwise.
\end{proposition}

For singleton queries, the query dimension recovers the
bandit DS dimension exactly.

\begin{proposition}
\label{prop:singleton-bds}
Let $\cQ_{\mathrm{sing}}:=\{\{y\}:y\in\cY\}$. For every nonempty
$C\sse\cY$,
\begin{equation}
\label{eq:singleton-rho}
    \rho_{\cQ_{\mathrm{sing}}}(C)=|C|-1.
\end{equation}
Consequently, $D_{\cQ_{\mathrm{sing}}}(\calH)=\operatorname{BDS}(\calH)$.
\end{proposition}

\subsection{Identification depth}
\label{sec:identification-depth}

For a nonempty $C\sse\cY$, the \emph{identification depth} $\tau_{\cQ}(C)$ is
the least depth of a $\cQ$-decision tree that completely refines $C$, with
value $\infty$ if no such tree exists. In particular, $\tau_{\cQ}(C)=0$ for a
singleton. Unlike the refinement cost, this quantity requires exact
identification of every candidate label. The following lemma shows that the identification depth of $C$ can be bounded in terms of the query dimension.

\begin{lemma}
\label{lem:rho-to-tau}
Let $D:=D_{\cQ}(\calH)<\infty$. For every $x\in\cX$ and every $C\sse \calH(x)$ with
$|C|\ge2$,
\[
    \tau_{\cQ}(C)=O(D\log|C|).
\]
\end{lemma}

\subsection{The extended response class}
\label{sec:extended-class}

For a fixed target, the response to a query is a binary function of the
instance--query pair. The upper-bound algorithm learns these functions on
$\cX\times\cQ$. Formally, we shall consider the following class of binary functions.

\begin{definition}[Extended response class]
\label{def:response-class}
For $h\in\calH$, define $r_h\colon\cX\times\cQ\to\{0,1\}$ by
\[
    r_h(x,A):=\indic{h(x)\in A}.
\]
The \emph{extended response class} of $\calH$ and its VC dimension are
\[
    \cR_{\cQ}(\calH):=\{r_h:h\in\calH\},
    \qquad
    d_{\cQ}(\calH):=\VC\bigl(\cR_{\cQ}(\calH)\bigr).
\]
\end{definition}

The next lemma shows that the VC dimension of this extended response class is upper bounded by the query dimension, up to logarithmic factors in $K$.

\begin{lemma}[VC dimension of the extended response class]
\label{lem:response-dim}
If $d_{\mathrm N}:=\Ndim(\calH)<\infty$, then
\[
    d_{\cQ}(\calH)=O(d_{\mathrm N}\log_2K).
\]
In particular, if $D_{\cQ}(\calH)<\infty$, then
\[
    d_{\cQ}(\calH)=O\bigl(D_{\cQ}(\calH)\log_2K\bigr).
\]
\end{lemma}

\subsection{Balanced query families}
\label{sec:examples}
\label{sec:balanced}

The general upper bound is quadratic in the query dimension.
We now show that this dependence becomes linear, up to logarithmic
factors, when the available queries can split every candidate set
into two parts of comparable size.

\begin{definition}[Balanced query family]
\label{def:balanced-queries}
Fix $\beta\in(0,1/2]$. A query family $\cQ$ is
\emph{$\beta$-balanced for $\calH$} if, for every $x\in\cX$ and
every $C\sse\calH(x)$ with $|C|\ge2$, some $A\in\cQ$ satisfies
\begin{equation}
\label{eq:balanced-split}
    \max\{|C\cap A|,\ |C\setminus A|\}
    \le (1-\beta)|C|.
\end{equation}
\end{definition}

Thus, regardless of the response, the query eliminates at least
a $\beta$ fraction of the candidate labels. This condition yields
the following sample-complexity guarantee.

\begin{theorem}[Balanced queries]
\label{thm:balanced-characterization}
Suppose that $\cQ$ is $\beta$-balanced for $\calH$ and that
$D:=D_{\cQ}(\calH)\in(0,\infty)$. For every
$0<\eps,\delta\le1/8$,
\begin{equation}
\label{eq:balanced-two-sided}
    \mQ(\calH;\eps,\delta)
    =\widetilde O_{\beta}\!\left(\frac D\eps\right).
\end{equation}
If $\calH$ is non-trivial in the sense of \Cref{def:nontrivial},
then $\widetilde O_{\beta}$ can be replaced by
$\widetilde\Theta_{\beta}$. The suppressed factors are
polylogarithmic in $K$ and $1/\delta$, and the implicit constants
may depend on $\beta$.
\end{theorem}

The improvement comes from the depth of the trees used for list
halving. The upper-bound algorithm requires only a tree whose
cells contain at most half the original candidates, rather than
one that identifies every label exactly. Under the balancing
condition, each split reduces the remaining candidate set by a
factor of at most $1-\beta$. Repeated splitting therefore produces
a halving tree of depth $O(1/\beta)$, independent of $D$.

Substituting this depth bound into the list-halving bound of
\Cref{lem:halving-depth-upper}, together with the response-class
bound of \Cref{lem:response-dim}, gives the upper bound linear
in $D$ up to logarithmic factors. The matching lower bound under
non-triviality follows from \Cref{thm:lower}.

\section{Algorithm and Analysis}
\label{sec:algorithm-analysis}

This section states and proves our main result: an upper bound and a lower bound on the sample complexity of learning with query feedback, both in terms of the query dimension $D:=D_{\cQ}(\calH)$. The upper bound is achieved by an algorithm that learns to shrink a list of candidate labels at every instance, and it holds whenever $D$ is finite. The lower bound requires in addition that the dimension be witnessed on more than one instance.

\begin{definition}[Non-trivial class]
\label{def:nontrivial}
The class $\calH$ is \emph{non-trivial} under $\cQ$ if $D_{\cQ}(\calH)$ is finite and is witnessed on at least two instances, that is, if there exist $d\ge2$, distinct instances $x_1,\dots,x_d\in\cX$, and a finite nonempty $W\sse\calH|_{x_{1:d}}$ with $r_{\cQ}(W)=D_{\cQ}(\calH)$.
\end{definition}

This condition is analogous to the non-triviality condition used by
\citet{bds}. It is required only for the lower bound.

\begin{theorem}[Sample-complexity upper bound]
\label{thm:quadratic-upper} Fix $\mathcal H$ and $\mathcal Q$ with $D:=D_{\cQ}(\calH)<\infty$. For every $\eps,\delta\in(0,1/2)$,
\begin{equation}
\label{eq:upper-bound}
    \mQ(\calH;\eps,\delta)
    =
    O\!\left(\frac{D\log^3K}{\eps}\Bigl(D\log K+\log\frac1\delta\Bigr)\right).
\end{equation}
\end{theorem}
In particular, $\mQ(\calH;\eps,\delta)=\widetilde O(D^2/\eps)$, where $\widetilde O$ hides factors polylogarithmic in $K$ and $1/\delta$. For comparison, we prove the following lower bound, which has the same dependence on $\eps$ and is linear in $D$.

\begin{theorem}[Sample-complexity lower bound]
\label{thm:lower} Let $\calH$ be non-trivial under $\cQ$ be non-trivial, and write
$D:=D_{\cQ}(\calH)$. For every $0<\eps,\delta\le1/8$,
\begin{equation}
\label{eq:lower-bound}
    \mQ(\calH;\eps,\delta)
    = \Omega\left(
    \frac{D+\log(1/\delta)}{\eps}\right).
\end{equation}
\end{theorem}
Together, \Cref{thm:quadratic-upper,thm:lower} prove \Cref{thm:intro-main}. The two bounds differ by a factor of order $D$, up to logarithmic factors.

\subsection{Proof of the upper bound}
\label{sec:quadratic-upper}

The learner maintains a nonempty list of candidate labels at every
instance and repeatedly halves these lists. Each stage learns to predict
query responses, then uses the predictor to choose which candidates to
retain. We first describe the algorithm, then prove that its query
distributions guarantee list reduction and that accurate response
prediction prevents the target from being discarded on too much
probability mass.

\paragraph{The list-halving algorithm.}
Write $D:=D_{\cQ}(\calH)<\infty$. If every $\calH(x)$ is a singleton,
the learner can output its unique element without any samples.
Otherwise $D\ge1$ and $\cQ\ne\emptyset$ by
\Cref{lem:rho-basic,lem:query-dimension-basic}; fix $A_0\in\cQ$.
Let $d:=d_{\cQ}(\calH)$, and fix an integer $\tau\ge1$ such that every
nonsingleton $C\sse\calH(x)$ admits a \emph{halving tree}: a
$\cQ$-decision tree whose leaves each contain at most $|C|/2$
candidates and whose depth is at most $\tau$.
By \Cref{lem:rho-to-tau,lem:response-dim}, we may take
\begin{equation}
\label{eq:list-halving-parameters}
    \tau=O(D\log K),
    \qquad d=O(D\log K).
\end{equation}
For each nonsingleton $C$, fix such a tree and pad every path to
length $\tau$ using $A_0$. Define a query distribution $\lambda_C$
by choosing a uniform candidate from $C$, choosing a uniform position
on its path, and returning the query at that position. For a singleton
$C$, let $\lambda_C$ be the point mass at $A_0$. 

To compare labels, represent $y\in\cY$ by its \emph{query trace}
$\operatorname{tr}_y(A):=\indic{y\in A}$. For binary functions
$u,v\colon\cQ\to\{0,1\}$ and a query distribution $\lambda$, write
\[
    L_\lambda(u,v):=\Pr_{A\sim\lambda}[u(A)\neq v(A)],
    \qquad
    L_\lambda(y,z):=L_\lambda(\operatorname{tr}_y,
                             \operatorname{tr}_z).
\]
This disagreement measure satisfies the triangle inequality.

\Cref{alg:trace-list-halving} gives the learner. Here
$\operatorname{BinaryPAC}$ is the learner for
$\cR_{\cQ}(\calH)$ supplied by \Cref{fact:binary-pac}. A call with
accuracy $\gamma$ and failure probability $\eta$ uses
$n = O((d+\log(1/\eta))/\gamma)$ examples.

\begin{algorithm2e}[H]
\DontPrintSemicolon
\caption{Learning with query feedback by list halving}
\label{alg:trace-list-halving}
\KwIn{$\calH$, $\cQ$, $\eps,\delta\in(0,1/2)$; query distributions
$\lambda_C$ defined above.}
$S\gets\lceil\log_2K\rceil$,\quad
$\gamma\gets\eps/(4\tau S)$,\quad $\eta\gets\delta/S$\;
$C_0(x)\gets\calH(x)$ for every $x\in\cX$\;
\For{$t=0,\ldots,S-1$}{
    $\lambda_{t,x}\gets\lambda_{C_t(x)}$ for every $x\in\cX$\;
    Collect $n$ fresh examples: at each $X_i\sim\cD$, draw
    $A_i\sim\lambda_{t,X_i}$ and request
    $b_i=\indic{h^*(X_i)\in A_i}$\;
    $g_t\gets\operatorname{BinaryPAC}
       (\{((X_i,A_i),b_i)\}_{i=1}^n;\gamma,\eta)$\;
    \ForEach{$x\in\cX$}{
        $f_t(x)\in\displaystyle\argmin_{y\in C_t(x)}
            L_{\lambda_{t,x}}(g_t(x,\cdot),\operatorname{tr}_y)$\;
        $C_{t+1}(x)\gets\{y\in C_t(x):
            L_{\lambda_{t,x}}(f_t(x),y)\le1/(2\tau)\}$\;
    }
}
\KwOut{$\widehat h(x)$, the unique element of $C_S(x)$, for every $x$.}
\end{algorithm2e}

At each stage, the learner selects the candidate whose responses best
match the binary predictor and retains a small neighborhood of that
candidate's trace. The two lemmas below explain why this update
halves each nonsingleton list while usually retaining the target.

\paragraph{From halving trees to single-query distributions.}
A halving tree tells us how several adaptive queries can reduce a
candidate set, but the learner may ask only one query per example.
We therefore use the tree to design a distribution rather than execute
it. The construction of $\lambda_C$ gives more weight to queries
reached by many candidates. Consequently, labels that separate before
the candidate set has halved must disagree on a query with significant
probability. This yields the following geometric guarantee.

\begin{lemma}[Queries from a halving tree]
\label{lem:halving-query-distribution}
Let $C\sse\cY$ be nonsingleton and admit a halving tree of depth at
most $\tau$. The distribution $\lambda_C$ constructed above satisfies
\begin{equation}
\label{eq:trace-small-balls}
    \left|\left\{z\in C:
        L_{\lambda_C}(y,z)\le\frac1{2\tau}
    \right\}\right|\le\frac{|C|}{2},
    \qquad y\in C.
\end{equation}
\end{lemma}

\begin{proof}
For a node $u$ of the padded tree, let $C_u$ be the candidates reaching
it and, if $u$ is internal, let $A_u$ be its query. By construction,
\begin{equation}
\label{eq:trace-query-distribution}
    \lambda_C(A)=\frac1{\tau|C|}
        \sum_{\substack{u\text{ internal}\\A_u=A}}|C_u|.
\end{equation}
These weights sum to one because every candidate visits exactly
$\tau$ internal nodes.

Fix $y\in C$, and let $u_y$ be the first node on its path with
$|C_{u_y}|\le|C|/2$. Such a node exists by the halving property.
Any $z\in C\setminus C_{u_y}$ diverges from $y$ at a strict ancestor $w$
of $u_y$. Thus $|C_w|>|C|/2$, and the query $A_w$ separates $y$
and $z$. It follows that
\[
    L_{\lambda_C}(y,z)
    \ge\lambda_C(A_w)
    \ge\frac{|C_w|}{\tau|C|}
    >\frac1{2\tau}.
\]
The ball in \eqref{eq:trace-small-balls} is therefore contained in
$C_{u_y}$, which has at most $|C|/2$ candidates.
\end{proof}

\paragraph{From response prediction to list halving.}
The preceding lemma guarantees that every nonsingleton list shrinks by
at least half, regardless of the chosen center. The remaining task is
to retain the target on most instances where it is still a candidate.
The next lemma shows that choosing the candidate trace closest to the
response predictor controls the probability mass of newly discarded
targets in terms of the predictor's average binary error.
\begin{lemma}[Halving the candidate lists]
\label{lem:one-step-list-halving}
Fix nonempty lists $C(x)\sse\calH(x)$, $x\in\cX$, and
$\eps',\delta'\in(0,1/2)$. One stage of
\Cref{alg:trace-list-halving}, with binary accuracy
$\gamma:=\eps'/(4\tau)$ and failure probability $\delta'$, uses
\begin{equation}
\label{eq:trace-one-stage-samples}
    O\!\left(\frac{\tau}{\eps'}
        \left(d+\log\frac1{\delta'}\right)\right)
\end{equation}
query-feedback examples and returns nonempty subsets
$C'(x)\sse C(x)$ satisfying
\[
    |C'(x)|\le\max\!\left\{1,\frac{|C(x)|}{2}\right\}.
\]
For every distribution $\cD$ and target $h^*\in\calH$, with
probability at least $1-\delta'$,
\begin{equation}
\label{eq:trace-one-stage-discard}
    \cD\bigl(\{x:h^*(x)\in C(x)\setminus C'(x)\}\bigr)
    \le\eps'.
\end{equation}
\end{lemma}

\begin{proof}
Set $\lambda_x:=\lambda_{C(x)}$, and let $\mu$ be the distribution
of $(X,A)$ obtained by drawing $X\sim\cD$ and $A\sim\lambda_X$.
The stage observes i.i.d.\ binary examples with label
$r_{h^*}(X,A)=\indic{h^*(X)\in A}$. By \Cref{fact:binary-pac}, the stated sample size suffices to return
$g\colon\cX\times\cQ\to\{0,1\}$ with binary error at most
$\gamma=\eps'/(4\tau)$, except with probability $\delta'$.
For $y\in\cY$, define
\[
    \ell_x(y):=L_{\lambda_x}
        \bigl(g(x,\cdot),\operatorname{tr}_y\bigr).
\]
On the binary learner's success event we have that
\begin{equation}
\label{eq:trace-average-error}
    \E_{X\sim\cD}[\ell_X(h^*(X))]
    =\Pr_{(X,A)\sim\mu}[g(X,A)\neq r_{h^*}(X,A)]
    \le\gamma.
\end{equation}

The predictor $g(x,\cdot)$ need not itself be a label trace. The
algorithm therefore chooses a closest candidate trace and keeps its
ball:
\begin{equation}
\label{eq:trace-list-update}
    f(x)\in\argmin_{y\in C(x)}\ell_x(y),
    \qquad
    C'(x):=\left\{y\in C(x):
        L_{\lambda_x}(f(x),y)\le\frac1{2\tau}\right\}.
\end{equation}
The new list contains $f(x)$, so is nonempty. Singleton lists remain
unchanged, and \Cref{lem:halving-query-distribution} halves every
other list, regardless of the predictor's accuracy.

To control target loss, consider any $x$ with $h^*(x)\in C(x)$.
The triangle inequality and the minimizing property of $f(x)$ give
\begin{equation}
\label{eq:trace-decoding-error}
    L_{\lambda_x}(f(x),h^*(x))
    \le\ell_x(f(x))+\ell_x(h^*(x))
    \le2\ell_x(h^*(x)).
\end{equation}
Thus the target survives whenever $\ell_x(h^*(x))\le1/(4\tau)$.
On the event \eqref{eq:trace-average-error}, Markov's inequality
now yields
\begin{align*}
    \cD\bigl(\{x:h^*(x)\in C(x)\setminus C'(x)\}\bigr)
    &\le\Pr_{X\sim\cD}\!\left[
        \ell_X(h^*(X))>\frac1{4\tau}\right]\\
    &\le4\tau\,\E_{X\sim\cD}[\ell_X(h^*(X))]
    \le4\tau\gamma=\eps'.
\end{align*}
\end{proof}

\paragraph{Completing the proof.}
We now iterate the one-stage guarantee. List shrinkage holds
deterministically; only the mass of newly discarded targets requires
a high-probability bound.

\begin{proof}[Proof of \Cref{thm:quadratic-upper}]
Fix $\cD$ and $h^*\in\calH$, and run
\Cref{alg:trace-list-halving} with
$S=\lceil\log_2K\rceil$. At each stage,
\Cref{lem:one-step-list-halving} applies with
$\eps'=\eps/S$ and $\delta'=\delta/S$.
The lists remain nonempty and satisfy
\[
    |C_t(x)|\le\max\{1,2^{-t}|\calH(x)|\},
    \qquad x\in\cX,\quad 0\le t\le S.
\]
Hence every final list $C_S(x)$ is a singleton, so the output
$\widehat h$ is well defined. For $0\le t<S$, let
\[
    E_t:=\{x:h^*(x)\in C_t(x)\setminus C_{t+1}(x)\}.
\]
Conditional on the history before stage $t$, the lists $C_t$ and
the query distributions are fixed and by \Cref{lem:one-step-list-halving} we have the probability that $\cD(E_t)>\eps/S$ is at most
$\delta/S$. Averaging over histories and taking a union bound,
with probability at least $1-\delta$,
$\cD(E_t)\le\eps/S$ for every stage.

The lists are nested and initially contain the target. Each incorrect
final prediction therefore corresponds to exactly one stage at which
the target was discarded: the sets $E_t$ are disjoint and their union
is the final error set. On the preceding success event,
\[
    L_{\cD}(\widehat h,h^*)
    =\sum_{t=0}^{S-1}\cD(E_t)\le\eps.
\]
Finally, summing \eqref{eq:trace-one-stage-samples} over the $S$
stages gives
\begin{align*}
    m
    &=O\!\left(\frac{\tau S^2}{\eps}
        \left(d+\log\frac S\delta\right)\right)\\
    &=O\!\left(\frac{D\log^3K}{\eps}
        \left(D\log K+\log\frac1\delta\right)\right),
\end{align*}
where we used \eqref{eq:list-halving-parameters},
$S=O(\log K)$, and $\log S=O(D\log K)$ since $D\ge1$.
This proves \eqref{eq:upper-bound}.
\end{proof}

\subsection{Proof of the lower bound}
\label{sec:lower}

In this section we prove \Cref{thm:lower} by constructing two hard cases, following the lower-bound argument of \citet{bds}. In the bandit setting, that argument rests on the fact that an unsuccessful prediction eliminates a single candidate label. A general query may eliminate many labels at once, so we instead bound the probability of a correct prediction at an instance directly in terms of the refinement cost of the candidate labels there (\Cref{lem:cell-success}). The bound concerns a single deterministic decision tree, whereas the learner's queries at one instance may depend on the feedback it received at others; \Cref{lem:contraction} shows that this adaptivity is immaterial. We then apply the bound to a distribution in which most samples fall on an uninformative point (\Cref{lem:anchored-lower}). The dependence on $\delta$ comes from a separate two-point construction.

\begin{lemma}
\label{lem:cell-success}
Let $C\sse\cY$ have $M\ge2$ elements, and let $T$ be a $\cQ$-decision tree of depth at most $n$ whose leaves are labeled by elements of $\cY$. Writing $T(y)$ for the label of the leaf reached by $y$,
\begin{equation}
\label{eq:cell-success}
    \Pr_{Y\sim\Unif(C)}\bigl[T(Y)=Y\bigr]
    \le
    \frac1M+\frac{n}{\rho_{\cQ}(C)}\cdot\frac{M-1}{M}.
\end{equation}
\end{lemma}

\begin{proof}
Let $q:=|\Pi_{T}(C)|$. Labels in the same cell reach the same leaf and receive the same output, so at most one label per cell is predicted correctly, and the probability in \eqref{eq:cell-success} is at most $q/M$. If $q=1$, this is the claim. If $q\ge2$, then $T$ is admissible in the infimum \eqref{eq:rho-def} defining $\rho_{\cQ}(C)$, so
\[
    \rho_{\cQ}(C)\le\depth(T)\,\frac{M-1}{q-1}\le n\,\frac{M-1}{q-1},
    \qquad\text{that is,}\qquad
    q\le1+\frac{n(M-1)}{\rho_{\cQ}(C)}.
\]
Dividing by $M$ proves \eqref{eq:cell-success}.
\end{proof}

The learner's queries at an instance may depend on the responses it has received at other instances. The next lemma shows that, once the target's labels at the other instances are fixed, the learner's prediction at an instance is computed by a decision tree whose depth is the number of visits to that instance.

\begin{lemma}
\label{lem:contraction}
Let $W\sse\calH|_{x_{1:d}}$ be a witness, and choose a representative
$h_v\in\calH$ for each $v\in W$. Suppose that a learner interacts with
target $h_U$, where $U$ is supported on $W$, and every sampled instance
belongs to $\{x_1,\dots,x_d\}$. Let $N_i$ be the number of occurrences
of $x_i$ in the sample. After fixing $U_{-i}$, the instance sequence,
and the learner's internal randomness, there is a deterministic $\cQ$-decision tree $T$ of depth at most $N_i$, with leaves labeled by elements of $\cY$, such that the learner's final prediction at $x_i$ equals $T(U_i)$.
\end{lemma}

\begin{proof}
We construct $T$ by simulating the learner on the fixed instance
sequence, maintaining a tree in which each node records a history. Initially,
the tree is a single root recording the empty history.
\begin{itemize}
  \item In a round at $x_j$ with $j\neq i$, we extend the history at each
  leaf by the learner's query and the response determined by $U_{-i}$. The tree itself
  does not change.
  \item In a round at $x_i$, each leaf $u$ has a query
  $q_u\in\cQ$ that the learner issues given the history at the leaf. We then
  attach to $u$ two children, one for each possible response to $q_u$. Each child records the history at $u$ extended by
  $q_u$ and that response.
\end{itemize}
After the last round, we label each leaf by the learner's final prediction at
$x_i$ given the history at that leaf. The tree branches only in the $N_i$ rounds at $x_i$, so has depth at most $N_i$, and by construction the learners prediction at $x_i$ is $T(U_i)$.
\end{proof}

We now choose the hard distribution. It places most of its mass on a fixed instance at which all candidate targets agree, and distributes the remaining mass among the other witness instances in proportion to their coordinate costs, so that the expected number of visits to each of them is small compared with the cost of its lines.

\begin{lemma}
\label{lem:anchored-lower}
Let $W\sse\calH|_{x_{1:d}}$ be a witness on $d\ge2$ instances with total cost $R:=r_{\cQ}(W)\in(0,\infty)$. For every $0<\eps\le1/8$,
\begin{equation}
\label{eq:anchored-lower}
    \mQ(\calH;\eps,1/8)\ge\frac{R}{128\eps}.
\end{equation}
\end{lemma}

\begin{proof}
Write $r_i:=r_i(W)$, and relabel the coordinates so that $r_1=\min_i r_i$. Since $r_1$
is minimal,
\[
    R_{-1}:=\sum_{i=2}^d r_i\ge\frac R2>0.
\]

Fix
$b$ such that the slice $W_b:=\{v\in W:v_1=b\}$ is nonempty. For $i\ge2$, every
$i$-line of $W$ meeting $W_b$ lies entirely in $W_b$, so every $i$-line of
$W_b$ is an $i$-line of $W$ and has refinement cost at least $r_i$.
Draw $U$ uniformly from $W_b$ and let $h_U\in\calH$ realize $U$. Independently
of $U$, draw the training instances from the distribution $\cD$ on
$x_1,\dots,x_d$ given by
\[
    \cD(x_1)=1-8\eps,
    \qquad
    \cD(x_i)=\frac{8\eps r_i}{R_{-1}},
    \quad 2\le i\le d,
\]
where $\eps\le1/8$ ensures $\cD(x_1)\ge0$.

Consider any (possibly randomized) learner that uses $m\le R_{-1}/(64\eps)$
rounds and outputs $\widehat h$. Let $N_i$ count the occurrences of $x_i$ in the
instance sequence, so that $\E N_i=m\,\cD(x_i)$. Fix $i\ge2$ with $r_i>0$.
Conditional on $U_{-i}$, the instance sequence, and the learner's random seed,
$U_i$ is uniform on the $i$-line through $U_{-i}$; as noted above, this line has
refinement cost at least $r_i$. By \Cref{lem:contraction}, the interaction then
reduces to a decision tree of depth at most $N_i$. Applying
\Cref{lem:cell-success} and averaging over the conditioning gives
\[
    \Pr[\widehat h(x_i)\neq U_i]
    \ge\frac12-\frac{\E N_i}{r_i}
    =\frac12-\frac{8\eps m}{R_{-1}}
    \ge\frac38.
\]
Thus, writing
\[
    Z:=\sum_{i=2}^d \cD(x_i)\indic{\widehat h(x_i)\neq U_i}
\]
for the error away from the first coordinate, we have $0\le Z\le8\eps$ and $\E Z\ge 3\eps$. Markov's
inequality yields
\[
    \Pr[Z\le\eps]
    =\Pr[8\eps-Z\ge7\eps]
    \le\frac{8\eps-\E Z}{7\eps}
    \le\frac57,
\]
so $\Pr[Z>\eps]\ge2/7>1/8$.

Since $h_U$ realizes $U$, the error of $\widehat h$ with respect to $h_U$ under
$\cD$ is at least $Z$. As the probability above averages over $U$, there is a
fixed $u\in W_b$ such that, with target $h_u$, the learner's error exceeds
$\eps$ with probability greater than $1/8$; that is, the learner fails the
$(\eps,1/8)$-PAC guarantee. Hence no $(\eps,1/8)$-PAC learner uses
$m\le R_{-1}/(64\eps)$ rounds, and \eqref{eq:anchored-lower} follows from
$R_{-1}\ge R/2$.
\end{proof}

\begin{proof}[Proof of \Cref{thm:lower}]
By non-triviality and \Cref{lem:anchored-lower},
\[
    \mQ(\calH;\eps,\delta)
    \ge\mQ(\calH;\eps,1/8)
    \ge\frac{D}{128\eps}.
\]
For the confidence dependence, non-triviality ensures that $\calH$
contains two hypotheses that agree at one instance and disagree at
another. Then a standard two-point lower bound (see the lower-bound proof
of \citet{bds}) yields
\[
    \mQ(\calH;\eps,\delta)
    \ge\Omega\!\left(\frac{\log(1/\delta)}{\eps}\right).
\]
Combining the two lower bounds proves the theorem.
\end{proof}

\section{Discussion}
In this work, we introduced the query dimension for multiclass PAC learning with
query feedback and establish a lower bound linear in this dimension
and an upper bound quadratic in it, up to logarithmic factors.
For balanced query families, the upper bound is linear and therefore
matches the lower bound up to logarithmic factors. The main open question is whether a
corresponding characterization holds for arbitrary query families.
Further directions include extending the analysis to the agnostic
setting and determining which logarithmic factors are necessary.

\paragraph{AI acknowledgment} In addition to literature review, we used GPT 6 to assist with the editing and drafting of this paper. Additionally, the proofs of Lemmas \ref{lem:rho-to-tau} and  \ref{lem:response-dim} are due to GPT 6. The authors verified the correctness and originality of all content in the paper and take full responsibility for its content.

\bibliography{references}

\appendix
\section{Learning-Theoretic Background}
\label{app:background}

\subsection{Combinatorial dimensions}
\label{app:dimensions}

We first recall the VC and Natarajan dimensions, which are used to
control the binary learning problems in the upper bound. We then
recall the DS and bandit DS dimensions, whose local witness structure
underlies the query dimension.

\begin{definition}[VC dimension]
\label{def:vc}
Let $\calH\sse\{0,1\}^{\mathcal Z}$ be a binary hypothesis class.
Distinct points $z_1,\dots,z_d\in\mathcal Z$ are \emph{shattered}
by $\calH$ if
\[
    \calH|_{z_{1:d}}=\{0,1\}^d.
\]
The \emph{VC dimension} $\VC(\calH)$ is the supremum of the sizes
of shattered sets.
\end{definition}

For multiclass hypotheses, the Natarajan dimension allows the two
labels used in the shattering condition to vary from one instance
to another.

\begin{definition}[Natarajan dimension]
\label{def:natarajan}
Let $\calH\sse\cY^{\cX}$. Distinct instances
$x_1,\dots,x_d\in\cX$ are \emph{Natarajan-shattered} by $\calH$
if there exist labels $a_i,b_i\in\cY$ with $a_i\neq b_i$ for
every $i\in[d]$ such that
\[
    \prod_{i=1}^d\{a_i,b_i\}
    \sse \calH|_{x_{1:d}}.
\]
The \emph{Natarajan dimension} $\Ndim(\calH)$ is the supremum
of the sizes of Natarajan-shattered sets.
\end{definition}

Both dimensions are defined to be zero if no nonempty set is
shattered and infinite if arbitrarily large finite sets are
shattered. When $\cY=\{0,1\}$, the two definitions coincide.

\paragraph{The DS dimension.}
A nonempty finite set $W\sse\cY^d$ is a \emph{pseudo-cube} if
$|L_i^W(v)|\ge2$ for every $v\in W$ and every $i\in[d]$.

\begin{definition}[DS dimension {\citep{daniely2014optimal,brukhim2022characterization}}]
\label{def:ds}
Let $\calH\sse\cY^{\cX}$. The \emph{DS dimension}
$\operatorname{DS}(\calH)$ is the supremum of the integers
$d\ge1$ for which there exist distinct instances
$x_1,\dots,x_d\in\cX$ such that $\calH|_{x_{1:d}}$ contains
a pseudo-cube. Its value is zero if no such witness exists
and infinite if such witnesses have arbitrarily many coordinates.
\end{definition}

\paragraph{The bandit DS dimension.}
\label{sec:bds-background}
The bandit DS dimension of \citet{bds} is defined through a local
condition on finite sets of label vectors. For $u,v\in\cY^d$ and
$i\in[d]$, we say that $u$ is an \emph{$i$-neighbor} of $v$ if
\[
    u_{-i}=v_{-i}
    \qquad\text{and}\qquad
    u_i\neq v_i.
\]
Thus, an $i$-neighbor provides an alternative label at coordinate
$i$ without changing any other coordinate. A nonempty finite set
$W\sse\cY^d$ is a \emph{pseudo-box with multiplicities
$N_1,\dots,N_d$} if every $v\in W$ has at least $N_i$ distinct
$i$-neighbors in $W$ for each $i\in[d]$. Here $N_1,\dots,N_d$
are positive integers. Following the formal definition in
\citet{bds}, the multiplicity $N_i$ counts neighbors, rather than
including the label of $v$ itself.

\begin{definition}[Bandit DS dimension {\citep{bds}}]
\label{def:bds}
Let $\calH\sse\cY^{\cX}$. Distinct instances $x_1,\dots,x_d$,
together with positive integers $N_1,\dots,N_d$, are
\emph{BDS-shattered} by $\calH$ if $\calH|_{x_{1:d}}$ contains
a pseudo-box with these multiplicities. The \emph{bandit DS
dimension} of $\calH$ is
\[
    \operatorname{BDS}(\calH)
    :=
    \sup\left\{
        \sum_{i=1}^d N_i:
        \begin{array}{l}
            d\ge1,\ x_1,\dots,x_d\in\cX\text{ distinct},\\
            N_1,\dots,N_d\in\mathbb N,\\
            (x_{1:d},N_{1:d})\text{ is BDS-shattered by }\calH
        \end{array}
    \right\}.
\]
The value is zero if no such witness exists and infinite if the
displayed sums are unbounded.
\end{definition}

Unlike Natarajan shattering, the pseudo-box condition does not
require a full Cartesian product of label choices. It requires
only that each vector have the prescribed number of alternatives
in every coordinate.

\subsection{Learning and growth bounds}
\label{sec:learning-background}

The upper-bound proof uses a realizable binary PAC learner at
each stage. The following guarantee is for the full-information
model, in which every training example is accompanied by its
binary label.

\begin{fact}[Realizable binary PAC learning {\citep{hanneke2016optimal}}]
\label{fact:binary-pac}
Let $\calH\sse\{0,1\}^{\mathcal Z}$ be a nonempty binary class
with $\VC(\calH)=d<\infty$. There is a universal constant
$C_{\mathrm{bin}}>0$ such that, for every
$\eps,\delta\in(0,1/2)$, a possibly improper learner achieves
error at most $\eps$ with probability at least $1-\delta$ using
\begin{equation}
\label{eq:binary-pac}
    \left\lceil
        C_{\mathrm{bin}}\frac{d+\log(1/\delta)}{\eps}
    \right\rceil
\end{equation}
i.i.d.\ labeled examples. The guarantee holds for every
distribution on $\mathcal Z$ and every target in $\calH$.
\end{fact}

To bound the VC dimensions of the binary classes arising in the
reduction, we use the following growth bound of
\citet{haussler-long}. It controls the number of label vectors
realized on finitely many instances in terms of the Natarajan
dimension.

\begin{fact}[Natarajan--Sauer growth bound]
\label{fact:natarajan-sauer}
Let $\calH\sse\cY^{\cX}$ have Natarajan dimension
$1\le n<\infty$, where $|\cY|=K$. For every $r\ge n$ and every
choice of distinct instances $x_1,\dots,x_r$,
\begin{equation}
\label{eq:natarajan-sauer}
    \bigl|\,\calH|_{x_{1:r}}\,\bigr|
    \le
    \left(\frac{erK^2}{n}\right)^n.
\end{equation}
\end{fact}

\section{Proofs for the Combinatorial Dimension}
\label{app:query-dimension-proofs}

\subsection{Elementary refinement bounds}
\label{app:rho-basic-proof}

\begin{proof}[Proof of \Cref{lem:rho-basic}]
Write $n:=|C|$. An admissible tree of depth $t$ inducing $q$
cells satisfies $2\le q\le n$ and $q\le2^t$. Hence
\[
    t\frac{n-1}{q-1}
    \ge
    (n-1)\frac{\log_2q}{q-1}
    \ge
    \log_2n,
\]
where the last inequality uses that
$u\mapsto\log_2u/(u-1)$ is decreasing for $u>1$. Taking the
infimum proves the lower bound. If a tree induces at least two
cells on $C$, some query separates two labels in $C$. The
depth-one tree for that query induces two nonempty cells and
has cost $n-1$, proving the upper bound. The two-label claim
follows.
\end{proof}

\subsection{Comparison with existing dimensions}
\label{app:dimension-comparison-proof}

\begin{proof}[Proof of \Cref{prop:dimension-comparison}]
In a pseudo-cube on $d$ coordinates, every line has at least two
labels and hence refinement cost at least one by
\Cref{lem:rho-basic}. Its query cost is therefore at least $d$,
proving $\operatorname{DS}(\calH)\le D_{\cQ}(\calH)$.

For the middle inequality, assume that $\cQ$ separates $\calH$
pointwise, as follows from finite query dimension by
\Cref{lem:query-dimension-basic}. Every nonsingleton subset of
$\calH(x)$ is split by some query, so \Cref{lem:rho-basic} gives
\[
    \rho_{\cQ}(C)\le|C|-1
    \qquad
    \text{for every nonempty }C\sse\calH(x).
\]
Fix a witness $W$ and set
\[
    m_i:=\min_{v\in W}\bigl(|L_i^W(v)|-1\bigr),
    \qquad
    I:=\{i\in[d]:m_i>0\}.
\]
Then $r_{\cQ}(W)\le\sum_i m_i$. If $I\neq\emptyset$, project
$W$ onto the coordinates in $I$. For each projected vector,
choose a preimage $v\in W$. For every $i\in I$, the
$i$-neighbors of $v$ project to distinct $i$-neighbors of the
projected vector. The projection is therefore a BDS witness
with multiplicities $(m_i)_{i\in I}$, giving
$\sum_i m_i\le\operatorname{BDS}(\calH)$. This inequality is
immediate if $I=\emptyset$. Taking the supremum over $W$ proves
$D_{\cQ}(\calH)\le\operatorname{BDS}(\calH)$.

Finally, a BDS witness with positive multiplicities
$N_1,\dots,N_d$ is also a pseudo-cube, so
$d\le\operatorname{DS}(\calH)$. Since $N_i\le K-1$ for every $i$,
\[
    \sum_{i=1}^dN_i
    \le d(K-1)
    \le(K-1)\operatorname{DS}(\calH).
\]
Taking the supremum proves the last inequality.
\end{proof}

\subsection{Singleton queries}
\label{app:singleton-bds}

\begin{proof}[Proof of \Cref{prop:singleton-bds}]
Under singleton queries, a positive response identifies the label,
so only the all-negative branch of a decision tree can continue to
split candidates. A depth-$t$ tree therefore induces at most $t+1$
nonempty cells on $C$, giving refinement cost at least $|C|-1$.
For $|C|\ge2$, one query $\{y\}$ with $y\in C$ attains this cost;
for a singleton, the cost is zero by definition. This proves
\eqref{eq:singleton-rho}.

Singleton queries separate all labels, so
\Cref{prop:dimension-comparison} gives
$D_{\cQ_{\mathrm{sing}}}(\calH)\le\operatorname{BDS}(\calH)$,
including when the dimensions are infinite. Conversely, a BDS
witness with multiplicities $N_1,\dots,N_d$ has at least $N_i+1$
labels in every $i$-line. Its query cost is therefore at least
$\sum_iN_i$ by \eqref{eq:singleton-rho}. Taking the supremum
proves the reverse inequality.
\end{proof}

\subsection{Identification depth}
\label{app:identification-depth-proof}

\begin{proof}[Proof of \Cref{lem:rho-to-tau}]
Fix $x\in\cX$. Let $H_j:=\sum_{r=1}^j1/r$ for $j\ge1$, with
$H_0:=0$. We prove by induction on $n:=|C|$ that every nonempty
$C\sse\calH(x)$ satisfies
\[
    \tau_{\cQ}(C)\le H_{n-1}D.
\]
The case $n=1$ follows from $\tau_{\cQ}(C)=0$. For $n\ge2$,
choose a tree $T$ attaining $\rho_{\cQ}(C)$, which is finite
by \Cref{lem:query-dimension-basic}. Let $C_1,\dots,C_k$ be
its nonempty cells, where $k\ge2$. Then
\[
    \depth(T)
    =\frac{k-1}{n-1}\rho_{\cQ}(C)
    \le\frac{k-1}{n-1}D.
\]
Each cell has at most $n-k+1$ labels and is contained in
$\calH(x)$, so the induction hypothesis gives
$\tau_{\cQ}(C_j)\le H_{n-k}D$ for every $j$. Attaching an
identification tree for each cell to its corresponding leaf yields
\[
\begin{aligned}
    \tau_{\cQ}(C)
    &\le\depth(T)+\max_j\tau_{\cQ}(C_j)\\
    &\le\left(\frac{k-1}{n-1}+H_{n-k}\right)D\\
    &\le H_{n-1}D.
\end{aligned}
\]
The last inequality follows from
\[
    H_{n-1}-H_{n-k}
    =\sum_{r=n-k+1}^{n-1}\frac1r
    \ge\frac{k-1}{n-1}.
\]
Since $H_{n-1}=O(\log n)$ for $n\ge2$, the result follows.
\end{proof}

\subsection{The extended response class}
\label{app:response-dim-proof}

\begin{proof}[Proof of \Cref{lem:response-dim}]
Suppose $q\ge1$ distinct pairs $(x_1,A_1),\dots,(x_q,A_q)$
are shattered by $\cR_{\cQ}(\calH)$, and let $z_1,\dots,z_r$
be the distinct instances among the $x_i$, so $r\le q$.
Each label vector on these instances determines all $q$
responses. Consequently,
\[
    2^q\le\bigl|\,\calH|_{z_{1:r}}\,\bigr|.
\]
If $r\le d_{\mathrm N}$, then $2^q\le K^r$, giving
$q\le d_{\mathrm N}\log_2K$. If $r>d_{\mathrm N}$, the
Natarajan--Sauer growth bound of \citet{haussler-long}
(\Cref{fact:natarajan-sauer}) gives
\[
    2^q
    \le\bigl|\,\calH|_{z_{1:r}}\,\bigr|
    \le\left(\frac{erK^2}{d_{\mathrm N}}\right)^{d_{\mathrm N}}
    \le\left(\frac{eqK^2}{d_{\mathrm N}}\right)^{d_{\mathrm N}}.
\]
Writing $a:=q/d_{\mathrm N}$ and taking logarithms yields
\[
    a\le\log_2e+2\log_2K+\log_2a.
\]
Since $a=O(\log_2K)$, every finite shattered set
has size $O(d_{\mathrm N}\log_2K)$. This proves the first claim.
The second follows from \Cref{prop:dimension-comparison}.
\end{proof}

\subsection{Balanced queries}
\label{app:balanced-queries}

We first record the dependence of the list-halving argument on the
depth of the available halving trees. This separates the algorithmic
bound from the particular estimate of that depth used in
\Cref{thm:quadratic-upper}.

\begin{lemma}[List halving with a uniform depth bound]
\label{lem:halving-depth-upper}
Suppose that, for some integer $b\ge1$, every nonsingleton
$C\sse\calH(x)$, $x\in\cX$, admits a $\cQ$-decision tree of depth
at most $b$ whose nonempty cells have size at most $|C|/2$. If
$d:=d_{\cQ}(\calH)<\infty$ and $S:=\lceil\log_2K\rceil$, then
for $\eps,\delta\in(0,1/2)$,
\begin{equation}
\label{eq:halving-depth-upper}
    \mQ(\calH;\eps,\delta)
    =O\!\left(
        \frac{bS^2}{\eps}
        \left[d+\log\frac S\delta\right]
    \right).
\end{equation}
\end{lemma}

\begin{proof}
When all the sets $\calH(x)$ are singletons, learning requires no
samples. Otherwise $\cQ\ne\emptyset$. By
\Cref{lem:halving-query-distribution}, a halving tree of depth at
most $b$ supplies a query distribution for which every ball of
radius $1/(2b)$ contains at most half the candidates. The proof
of \Cref{lem:one-step-list-halving} uses its depth parameter only
through this property. With $b$ in place of $\tau$, that proof
gives a list-halving step using
\[
    O\!\left(\frac b\eps'
        \left[d+\log\frac1\delta'\right]\right)
\]
samples and, with probability at least $1-\delta'$, discards the
target on probability mass at most $\eps'$.

Initialize the lists to $\calH(x)$ and apply $S$ such steps, each
using fresh samples and parameters $\eps'=\eps/S$ and
$\delta'=\delta/S$. The final lists are singletons. A union bound
over the conditional failure probabilities shows that, with
probability at least $1-\delta$, the total probability mass on
which the target has been discarded is at most $S\eps'=\eps$.
Summing the sample counts proves \eqref{eq:halving-depth-upper}.
\end{proof}

\begin{proof}[Proof of \Cref{thm:balanced-characterization}]
Set
\[
    b:=\left\lceil\frac{\ln2}{-\ln(1-\beta)}\right\rceil.
\]
For any nonsingleton $C\sse\calH(x)$, recursively apply balanced
splits, stopping each branch when its candidate set has size at
most $|C|/2$. After $t$ queries, any branch that has not stopped
contains at most $(1-\beta)^t|C|$ candidates. Since
$(1-\beta)^b\le1/2$, this gives a halving tree of depth at most
$b=O(1/\beta)$.

By \Cref{lem:halving-depth-upper,lem:response-dim}, with
$S:=\lceil\log_2K\rceil$,
\begin{equation}
\label{eq:balanced-upper}
    \mQ(\calH;\eps,\delta)
    =O\!\left(
        \frac{S^2}{\beta\eps}
        \left[D\log(2K)+\log\frac S\delta\right]
    \right).
\end{equation}
Since $D>0$ implies $D\ge1$, this is
$\widetilde O_\beta(D/\eps)$ when factors polylogarithmic in $K$
and $1/\delta$ are suppressed. If $\calH$ is non-trivial,
\Cref{thm:lower} gives
$\Omega((D+\log(1/\delta))/\eps)$ for
$0<\eps,\delta\le1/8$, proving the characterization.
\end{proof}

\end{document}